\documentclass[11pt]{article}

\usepackage[T1]{fontenc}
\usepackage[utf8]{inputenc}
\usepackage{lmodern}
\usepackage[margin=1in]{geometry}
\usepackage{microtype}
\usepackage{graphicx}
\usepackage{booktabs}
\usepackage{amsmath}
\usepackage{amssymb}
\usepackage{amsthm}
\usepackage{natbib}
\usepackage[font=small,labelfont=bf]{caption}
\PassOptionsToPackage{hyphens}{url}
\usepackage[hidelinks]{hyperref}
\theoremstyle{plain}
\newtheorem{lemma}{Lemma}
\theoremstyle{definition}
\newtheorem{assumption}{Assumption}

\newcommand{\R}{\mathbb{R}}
\newcommand{\dff}{d_{\mathrm{ff}}}
\newcommand{\dint}{d_{\mathrm{int}}}
\newcommand{\What}{\widehat W}
\newcommand{\Wtop}{\widehat W^{\mathrm{top}}}
\newcommand{\Wmet}{\widehat W^{\mathrm{met}}}
\newcommand{\Wsh}{\widehat W^{\mathrm{sh}}}
\newcommand{\PD}{\mathrm{PD}}
\newcommand{\GW}{\mathrm{GW}}
\newcommand{\Norm}{\mathrm{Norm}}
\newcommand{\FFN}{\mathrm{FFN}}

\newcommand{\GromovW}{Gromov-\allowbreak Wasserstein}

\title{Geometry-Guided Layerwise FFN Width Allocation in Transformers}
\author{%
Timur Mudarisov$^{1}$ \qquad
Mikhail Burtsev$^{2}$ \qquad
Radu State$^{1}$\\[0.75em]
\small $^{1}$University of Luxembourg\\
\small $^{2}$London Institute of Mathematical Sciences
}
\date{}
\hypersetup{
  pdftitle={Geometry-Guided Layerwise FFN Width Allocation in Transformers},
  pdfauthor={Timur Mudarisov, Mikhail Burtsev, Radu State}
}

\begin{document}
\maketitle

\begin{abstract}
Feed-forward networks (FFNs) account for a large fraction of Transformer parameters, yet their hidden width is usually constant across depth. We ask whether this capacity can instead be allocated from a forward-pass measurement of layer behavior. We view each FFN as transporting a cloud of token representations and quantify the induced geometric change using correspondence-preserving shift, Gromov-Wasserstein distortion, and degree-one persistent homology under raw and scale-normalized metrics. A layerwise approximation surrogate yields an exact fixed-budget optimizer. Across seven pretrained language models, raw Euclidean work largely tracks residual-norm growth, whereas normalized work is predominantly front-loaded. Gromov-Wasserstein work is more consistently associated with perturbation-based layer sensitivity than the finite-sample topological estimate. In paired 128M and 256M training runs, several normalized-work schedules reduce mean validation loss relative to both uniform width and a hand-designed cosine taper. With the amplified paired differences at 440M, the best geometry-based allocations improve over uniform substantially larger than the cosine taper, while the anti-topological raw control is worse than uniform.
\end{abstract}

\section{Introduction}
\label{sec:intro}
Feed-forward networks are the principal pointwise computation inside a
Transformer block and typically dominate its parameter count
\citep{vaswani2017attention}. Nevertheless, most decoder-only language models
use one global intermediate dimension: every layer receives the same $\dff$.
This design is simple, but it assumes that the amount of useful nonlinear
capacity is uniform even though representations and layer functions change with
depth \citep{elhage2021framework,jastrzebski2018residual}. Uniform width is
therefore an architectural convention rather than an evident optimum.

Tapered Language Models (TLMs) provide direct evidence that this convention can
be improved \citep{tlm2026}. Their 440M Transformer uses the same total
parameter and FLOP budget while reallocating MLP width from later to earlier
layers. A smooth cosine taper reduces their in-distribution validation
perplexity from $16.28$ to $14.44$, and the selected schedule is transferred to
larger scales and other token-mixing architectures. This result establishes
that depth-aware allocation matters. It does not, however, determine the width
of an individual layer from the computation performed at that layer: the cosine
shape and its endpoints remain hand-designed hyperparameters.

A prescribed taper and a measured allocation answer different questions. A
taper asks which low-dimensional family of depth profiles should be searched;
once the family is selected, every model of a given depth receives essentially
the same shape. A measured rule asks which layers of this particular model and
data distribution exhibit the strongest evidence of approximation demand. It
can reproduce a monotone taper when the evidence is monotone, but it can also
retain local peaks or plateaus that a cosine necessarily smooths away. This
distinction matters for transfer: a hand-designed schedule is cheap and does
not require a reference checkpoint, whereas an activation-derived schedule can
adapt to architecture, scale, and pretraining distribution. Our experiments
therefore use cosine as a strong structural baseline.

We ask whether activations of a reference model can provide a data-dependent
schedule. At FFN layer $\ell$, the residual states of a sequence form a point
cloud $Z_\ell$, and the residual update maps it to a cloud $H_\ell$. We call the
change between the two clouds the layer's \emph{geometric work}. The definition
must be aligned with the FFN readout. LayerNorm and RMSNorm largely remove
positive per-token rescaling before the branch is evaluated
\citep{ba2016layer,zhang2019root,xiong2020layernorm}. Raw pairwise-distance and
persistence statistics are therefore directly scale-sensitive. By contrast,
the raw correspondence-preserving shift equals the magnitude of the FFN branch
output and is not algebraically multiplied by the residual norm, although the
two quantities may covary empirically across depth. We compare raw Euclidean
geometry with spherical and hyperbolic metrics built from unit-normalized
states, and distinguish pointwise motion from changes in relational and
topological structure.

To turn a work profile into widths, we introduce a layerwise approximation
surrogate. A latent coefficient $C_\ell$ represents how difficult the FFN map is
to approximate on the representation distribution, $\lambda_\ell$ represents
the loss consequence of an approximation error, and $\dint(\ell)$ controls the
error-width exponent. A constrained optimizer then assigns the fixed total FFN
budget. Geometric work $\What_\ell$ is used as a proxy for $C_\ell$. This proxy
is a modeling hypothesis rather than a consequence of topology or optimal
transport: the present experiments diagnose its ingredients and test the
resulting schedules end to end, but do not directly fit layerwise error-width
curves.

The paper makes four contributions. First, it defines three complementary
forward-pass work statistics under three geometries, separating raw magnitude
from scale-normalized reshaping. Second, it derives the unique continuous
budget-optimal width allocation, including lower bounds and the
depth-dependent-exponent case in which a single proportionality constant is not
valid; hardware rounding is treated separately and preserves the budget but is
not claimed to solve the discrete optimization problem. Third, it evaluates
scale sensitivity, local loss additivity, intrinsic dimension, sensitivity
association, and work profiles across seven pretrained decoders. Fourth, it
compares geometry-derived schedules with uniform width, a TLM-style cosine
taper, and an anti-work control at fixed parameter count. The strongest
experiment is the five-seed 440M comparison, where topological/hyperbolic
allocation has a substantially larger mean within-protocol loss reduction than
cosine.

\section{Theory}
\label{sec:theory}

\subsection{FFN Sublayers as Token Transport}
A decoder block processes $N$ token states in $\R^d$. Let
$Z_\ell=\{z_\ell^{(i)}\}_{i=1}^{N}$ denote the post-attention residual cloud at
block $\ell$. 
The FFN update is
\begin{equation}
\label{eq:ffn-transport}
h_\ell^{(i)}=z_\ell^{(i)}+
\FFN_\ell\!\left(\Norm(z_\ell^{(i)})\right),
\qquad
H_\ell=\{h_\ell^{(i)}\}_{i=1}^{N}.
\end{equation}
Thus the network can be viewed as a sequence of residual transports
\citep{jastrzebski2018residual}. For an embedding
$\pi:\R^d\rightarrow(\mathcal M,\rho)$ and a cloud comparison $D$, define
\begin{equation}
\label{eq:work-general}
\What_\ell=D\!\left(\pi(Z_\ell),\pi(H_\ell)\right).
\end{equation}
For positive $c$, LayerNorm and RMSNorm satisfy
$\Norm(cz)\simeq\Norm(z)$ up to the numerical stabilizer. This invariance applies
to the branch input rather than to the full residual state, but it motivates a
readout-aligned criterion: a work statistic should not be dominated by
variations that the FFN input normalization largely removes. Raw work is
therefore retained as a control, while normalized metrics are the primary
candidates for allocation.

\subsection{Geometries and Work Statistics}
We evaluate the same input/output clouds in three metric spaces. For a nonzero
state define $u(z)=z/\lVert z\rVert$. \emph{Raw geometry} uses the identity map
and Euclidean distance. \emph{Spherical geometry} uses $u(z)$ and angular
distance
\begin{equation}
\rho_{\mathrm{sph}}(z,z')=
\arccos\!\left(\left\langle u(z),u(z')\right\rangle\right).
\end{equation}
This unit normalization is a scale-invariant proxy inspired by branch-input
normalization but it is not identical to full LayerNorm, which recenters the vector
and may apply learned coordinatewise gains.

\emph{Hyperbolic geometry} uses the unit Poincar\'e ball
$\mathbb B^d=\{x\in\R^d:\lVert x\rVert<1\}$ with curvature $-1$ and ball
radius $R=1$. Because $u(z)$ lies on the boundary, we clip it one percent inward:
\begin{equation}
\label{eq:hyp-map}
\pi_{\mathrm{hyp}}(z)=\rho_c u(z),
\qquad \rho_c=1-0.01=0.99.
\end{equation}
The geodesic distance is
\begin{equation}
\label{eq:hyp-distance}
d_{\mathbb B}(x,y)=\operatorname{arcosh}\!\left(
1+\frac{2\lVert x-y\rVert^2}
{(1-\lVert x\rVert^2)(1-\lVert y\rVert^2)}\right).
\end{equation}
All embedded points therefore have the same radial coordinate $0.99$. If
$\alpha$ is their spherical angle, then
\begin{equation}
\label{eq:hyp-angular}
d_{\mathbb B}(\rho_c u,\rho_c v)=
\operatorname{arcosh}\!\left(
1+\frac{8\rho_c^2\sin^2(\alpha/2)}{(1-\rho_c^2)^2}\right).
\end{equation}
Thus this hyperbolic construction is a fixed nonlinear monotone transform of
angular separation, and it does not introduce a learned radial hierarchy. We use it
as an empirical reweighting of angular distances rather than as evidence that
the token cloud is intrinsically hierarchical.

For each geometry, we compute three notions of work. The first keeps the known
token correspondence and measures mean shift,
\begin{equation}
\label{eq:wshift}
\Wsh_\ell=\frac{1}{N}\sum_{i=1}^{N}
\rho\!\left(\pi z_\ell^{(i)},\pi h_\ell^{(i)}\right).
\end{equation}
This statistic is simple and directly measures update magnitude, but it does not
distinguish a nearly rigid transport from a change in the internal shape of the
cloud. In raw geometry, Eq.~\eqref{eq:ffn-transport} gives
\begin{equation}
\label{eq:raw-shift-identity}
\Wsh_\ell=\frac{1}{N}\sum_{i=1}^N
\left\lVert\FFN_\ell\!\left(\Norm(z_\ell^{(i)})\right)\right\rVert.
\end{equation}
Consequently, under $\Norm(cz)\simeq\Norm(z)$, positive rescaling of the
residual state does not algebraically rescale raw mean shift for a fixed FFN.
Any depthwise association between raw shift and residual norm is therefore an
empirical co-variation of branch-output magnitude, not a direct metric-scale
confound.

The second statistic measures metric distortion. With uniform token weights
$\mu_i=1/N$, let $A_{ij}=\rho(z_i,z_j)$ and $B_{km}=\rho(h_k,h_m)$. We use the
entropically regularized Sinkhorn approximation to the quadratic
\GromovW{} objective $\Wmet_\ell$,
\begin{equation}
\label{eq:wmet}
\begin{aligned}
&\GW_{2,\varepsilon}^{2,\mathrm{Sink}}(A,B):=\min_{T\in\Pi(\mu,\mu)}
\sum_{i,j,k,m}\lvert A_{ij}-B_{km}\rvert^2T_{ik}T_{jm} \\
&+\varepsilon_{\mathrm{GW}}\,
\mathrm{KL}\!\left(T\,\middle\|\,\mu\otimes\mu\right),
\end{aligned}
\end{equation}
where $\varepsilon_{\mathrm{GW}}>0$ is fixed across layers, geometries, and
models, and the coupling is optimized by Sinkhorn iterations
\citep{memoli2011gromov,peyre2019computational}. We report the regularized
quadratic objective itself: Eq.~\eqref{eq:wmet} is a squared-discrepancy
convention, not its square root and not a debiased Sinkhorn divergence. The
optimized coupling treats the clouds as unlabeled metric-measure spaces and
asks whether pairwise relations can be preserved after allowing a coupling,
whereas Eq.~\eqref{eq:wshift} is the correspondence-preserving baseline.

The third statistic focuses on ordinary degree-one persistent homology with
coefficients in $\mathbb F_2$. For a finite metric cloud, the closed
Vietoris--Rips filtration includes a simplex at scale $t$ when all its pairwise
distances are at most $t$. We compute the full finite $H_1$ filtration up to the
cloud diameter, without filtration clipping or truncation. Ordinary and reduced
homology coincide in degree one. Let $\PD_1(Z_\ell)$ and $\PD_1(H_\ell)$ denote
the resulting finite diagrams, with the diagonal available with infinite
multiplicity. We define
\begin{equation}
\label{eq:wtop}
\Wtop_\ell=
\mathcal W_2\!\left(\PD_1(Z_\ell),\PD_1(H_\ell)\right),
\end{equation}
where $\mathcal W_2$ is the unsquared $2$-Wasserstein distance using the
$\ell_\infty$ ground metric in the birth--death plane:
\begin{equation}
\mathcal W_2(P,Q)=
\left[\inf_{\gamma}\sum_{p\in P}
\lVert p-\gamma(p)\rVert_\infty^2\right]^{1/2}.
\end{equation}
The bijection $\gamma$ may match points to the diagonal, so an empty diagram is
handled by matching every off-diagonal point of the other diagram to its nearest
diagonal point. This statistic responds to the creation or destruction of
loops, not to component merging, which belongs to $H_0$
\citep{edelsbrunner2010computational,boissonnat2018geometric}. Classical
Vietoris--Rips stability is naturally stated using interleaving/Gromov--Hausdorff
control and bottleneck distance. Here $\mathcal W_2$ is used as a finite-sample
summary on finite diagrams rather than claimed to satisfy an unrestricted
universal stability theorem \citep{cohensteiner2007stability}. 

\subsection{Approximation Surrogate}
Let $w_\ell=\dff(\ell)$ be the width assigned to FFN layer $\ell$, let
$f_\ell(z)=\FFN_\ell(\Norm(z))$, and let $\widetilde f_{\ell,w}$ denote a
width-$w$ approximation to the trained branch map. The approximation error in
the surrogate is the expected squared branch-output error on the calibration
distribution $P_\ell$,
\begin{equation}
\label{eq:epsilon-definition}
\epsilon_\ell(w)=
\mathbb E_{z\sim P_\ell}\!\left[
\lVert f_\ell(z)-\widetilde f_{\ell,w}(z)\rVert_2^2\right].
\end{equation}
Thus $C_\ell$ has the units of squared branch-output error times
$w^{a_\ell}$, while $\lambda_\ell$ converts that local error to a loss effect.
We separate the latent approximation coefficient $C_\ell$ from its observed
geometric proxy $\What_\ell$ through three assumptions.

\begin{assumption}[Local additive loss]
\label{ass:additive}
Near a trained solution, the loss effect of layerwise approximation errors is
locally additive:
$\Delta\mathcal L\approx\sum_\ell\lambda_\ell\epsilon_\ell$, where
$\lambda_\ell\geq0$ is the sensitivity of layer $\ell$.
\end{assumption}

\begin{assumption}[Intrinsic error-width rate]
\label{ass:rate}
The residual cloud at layer $\ell$ concentrates near a set of intrinsic
dimension $\dint(\ell)\ll d$, and the FFN approximation error obeys
\begin{equation}
\label{eq:rate}
\epsilon_\ell(w_\ell)\approx
C_\ell w_\ell^{-a_\ell},
\qquad
a_\ell=\frac{s}{\dint(\ell)}>0,
\end{equation}
where $s$ is a global expansion-order parameter. All experiments use the fixed
value $s=3$.
\end{assumption}
This form follows the dependence of nonparametric approximation rates on
intrinsic rather than ambient dimension
\citep{gyorfi2002distribution,tsybakov2009introduction,yarotsky2017error,
chen2019manifold,nakada2020intrinsic}. It is an ansatz for Transformer FFNs, not
a theorem about their learned functions.

\begin{assumption}[Geometric proxy]
\label{ass:proxy}
On the distribution of residual states used for calibration, the measured work
$\What_\ell$ is informative about $C_\ell$, so that substituting
$C_\ell\leftarrow\What_\ell$ preserves enough of the layerwise ordering to
construct a useful allocation.
\end{assumption}
A smooth map can alter the observed cloud substantially, and a complex map can
approximately preserve it, so Assumption~\ref{ass:proxy} is necessarily
model- and distribution-dependent. A direct test would fit error-width curves
for individual layers. Without that experiment, correlations with sensitivity
and end-to-end training are indirect evidence only; moreover, $C_\ell$ and
$\lambda_\ell$ are distinct quantities and need not be strongly correlated.

\subsection{Budget-Optimal Width}
Combining Assumptions~\ref{ass:additive} and \ref{ass:rate} gives the
lower-bounded continuous surrogate
\begin{equation}
\label{eq:program}
\min_{w_\ell\geq w_{\min}}
\sum_{\ell=1}^{L}\lambda_\ell C_\ell w_\ell^{-a_\ell}
\quad\text{subject to}\quad
\sum_{\ell=1}^{L}w_\ell=B.
\end{equation}
The feasibility condition is $B\geq Lw_{\min}$. When hidden dimension and FFN
parameterization are fixed across layers, the width budget is proportional to
both FFN parameters and FFN FLOPs.

\begin{lemma}[Lower-bounded budget-optimal FFN width]
\label{lem:width}
Let $c_\ell=\lambda_\ell C_\ell\geq0$, $a_\ell>0$, $w_{\min}>0$, and
$B\geq Lw_{\min}$. If $B>Lw_{\min}$, assume that at least one $c_\ell>0$.
Then Problem~\eqref{eq:program} has the unique continuous minimizer
\begin{equation}
\label{eq:kkt-law}
w_\ell^\star=
\max\!\left\{w_{\min},
\left(\frac{a_\ell c_\ell}{\mu^\star}\right)^{1/(a_\ell+1)}\right\}.
\end{equation}
For $B>Lw_{\min}$, $\mu^\star>0$ is the unique value satisfying
$\sum_\ell w_\ell^\star=B$. At $B=Lw_{\min}$, the allocation
$w_\ell^\star=w_{\min}$ is unique, although the KKT multiplier need not be.
If every $c_\ell>0$ and the lower bound is inactive, Eq.~\eqref{eq:kkt-law}
reduces to
\begin{equation}
\label{eq:general-law}
w_\ell^\star=
\left(\frac{a_\ell c_\ell}{\mu^\star}\right)^{1/(a_\ell+1)}.
\end{equation}
If additionally $a_\ell=a$ for all layers and $\theta=1/(a+1)$, then
\begin{equation}
\label{eq:closed-law}
w_\ell^\star=
B\frac{(\lambda_\ell C_\ell)^\theta}
{\sum_k(\lambda_k C_k)^\theta}.
\end{equation}
\end{lemma}

\begin{proof}[Proof sketch]
The feasible set is compact when $B\geq Lw_{\min}$, and the objective is
continuous and convex. If $B>Lw_{\min}$ and at least one coefficient is
positive, every zero-coefficient layer is optimally placed at $w_{\min}$;
otherwise transferring width from it to a positive-coefficient layer strictly
reduces the objective. The objective is strictly convex on the remaining
positive-coefficient coordinates, which gives a unique allocation. The KKT
condition for an active coordinate is
$-a_\ell c_\ell w_\ell^{-a_\ell-1}+\mu=0$; an inactive coordinate is clamped at
the lower bound, yielding Eq.~\eqref{eq:kkt-law}. For $B>Lw_{\min}$, the sum of
the right-hand side is continuous and strictly decreasing at the budget level,
so it determines a unique $\mu^\star$. At $B=Lw_{\min}$ the feasible set is a
single point. If all $c_\ell=0$ and $B>Lw_{\min}$, the objective is constant and
the optimizer is not unique; the practical rule then falls back to uniform
allocation.
\end{proof}

With $a_\ell=s/\dint(\ell)$ and $s=3$, define
$\theta_\ell=\dint(\ell)/(3+\dint(\ell))$. For depth-dependent intrinsic
dimension, Eq.~\eqref{eq:kkt-law} must be solved with a shared multiplier; one
cannot in general write
$w_\ell\propto(\lambda_\ell C_\ell)^{\theta_\ell}$ using a single
layer-independent proportionality constant because the multiplier is raised to
different powers. One-dimensional bisection finds $\mu^\star$.

Equation~\eqref{eq:closed-law} also clarifies the role of the exponent. As
$\theta\rightarrow0$, the optimizer approaches uniform width even when work
varies strongly; as $\theta\rightarrow1$, width becomes nearly proportional
to the sensitivity-weighted complexity. Intermediate values compress noisy work
estimates. Uniform allocation is thus a limiting strongly regularized case,
while a hand-designed taper can be viewed as replacing
$(\lambda_\ell C_\ell)^\theta$ by a prescribed depth prior. The
geometry-guided schedule instead estimates that factor from activations and
applies the exact continuous budget normalization.
The continuous solution is subsequently rounded to hardware-compatible
multiples and rebalanced to recover the exact total parameter budget.

\subsection{From a Reference Model to a Schedule}
The practical geometry-guided rule substitutes
$C_\ell\leftarrow\What_\ell$. The experiments use the pure-work variant
$\lambda_\ell=1$ and the fixed expansion-order parameter $s=3$; perturbation
sensitivity is kept independent so that it can serve as a diagnostic rather
than a tuning signal. A zero work value is admissible because the continuous
program includes $w_{\min}>0$: the corresponding layer is assigned the floor by
Eq.~\eqref{eq:kkt-law}. If an entire profile is identically zero, the practical
fallback is uniform allocation.

A schedule is constructed by (1) collecting FFN input/output clouds from a
uniform-width reference model, (2) estimating and aggregating the work profile,
(3) interpolating that profile to the target depth when needed, (4) solving the
lower-bounded continuous fixed-budget program, and (5) rounding and rebalancing
the widths. The target model is then trained from scratch. This procedure
assumes that the relative layer roles measured in the reference model remain
informative after the architecture is reallocated.

\section{Experiments}
\label{sec:experiments}

\subsection{Experimental Protocols}
\paragraph{Pretrained-model diagnostics.}
We analyze seven decoder-only checkpoints: Llama 3.2 at 1B and 3B parameters and
Llama 3.1 at 8B \citep{dubey2024llama}; Gemma 2 at 2B and 9B
\citep{gemmateam2024gemma2}; Gemma 3 at 1B \citep{gemmateam2025gemma3}; and
Mistral 7B v0.3 \citep{jiang2023mistral}. A held-out OpenWebText sample
\citep{gokaslan2019openwebtext} is passed through each model, and the post-attention
FFN input and output clouds are recorded. Raw statistics use the original
residual states. Spherical and hyperbolic statistics use the corresponding
normalized embeddings. Persistence and \GromovW{} computations use the same
fixed token subsample so that differences are not caused by different examples.

\paragraph{Perturbation diagnostics.}
Layer sensitivity is estimated by low-rank FFN perturbations. For each layer we
remove increasing rank fractions and record the validation-loss increase; the
local slope of the small-perturbation region defines $\lambda_\ell$. To probe
Assumption~\ref{ass:additive}, sets of layers are perturbed jointly and their
loss increase is compared with the sum of the corresponding single-layer
increases.

Intrinsic dimension is estimated independently at each relative depth using
TwoNN \citep{facco2017twonn}. For every point $x_i$ in the layer-$\ell$ point
cloud, let $r_{i,1}$ and $r_{i,2}$ denote its distances to the first and second
nearest neighbors, and define
\[
    \mu_i = \frac{r_{i,2}}{r_{i,1}}.
\]
Under the local-uniformity model of TwoNN, the cumulative distribution satisfies
\[
    F(\mu)=1-\mu^{-d}.
\]
After sorting the ratios, we estimate $\dint(\ell)$ as the slope through the
origin in
\[
    -\log\!\left(1-\widehat F(\mu_{(i)})\right)
    \approx
    \dint(\ell)\log\mu_{(i)}.
\]
Equivalently, writing
$x_i=\log\mu_{(i)}$ and
$y_i=-\log(1-\widehat F(\mu_{(i)}))$, the fitted slope is
\[
    \widehat{\dint}(\ell)
    =
    \frac{\sum_{i\in\mathcal I}x_i y_i}
         {\sum_{i\in\mathcal I}x_i^2},
\]
where $\mathcal I$ and the empirical-CDF plotting positions are specified
below. The resulting TwoNN estimates are used directly, without clipping,
winsorization, or smoothing, and the fixed value $s=3$ gives
$a_\ell=3/\widehat{\dint}(\ell)$. These measurements do not train alternative
architectures and therefore provide inexpensive diagnostics of the surrogate.
\paragraph{From-scratch allocation runs.}
We train approximately 128M and 256M Transformers on byte-level-BPE-tokenized
OpenWebText \citep{radford2019language}. The comparison includes all nine
work/geometry combinations, uniform width, the TLM-style cosine taper, and an
anti-topological/raw control. Each size uses three paired seeds: within a seed,
all rules share initialization and data order. Every architecture has the same
total parameter count, token budget, and AdamW configuration, and final
validation loss is the primary endpoint. The 440M experiment allocates the
larger compute budget to five paired seeds for a focused set of rules: uniform,
cosine, topological/hyperbolic, Gromov/spherical, and anti-topological/raw.

\begin{figure*}[t]
\centering
\includegraphics[width=\linewidth]{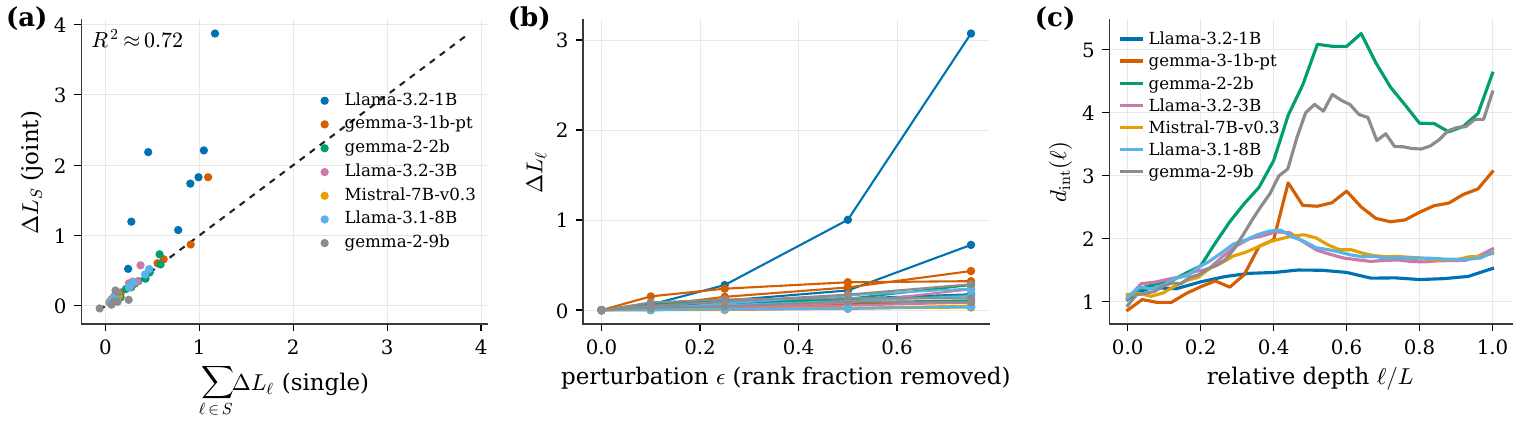}
\caption{\textbf{Diagnostics for the allocation surrogate.}
\textbf{(a)} Joint perturbation losses are associated with the sum of the
corresponding single-layer losses. \textbf{(b)} Small rank-removal
perturbations are approximately linear in the removed rank fraction.
\textbf{(c)} TwoNN intrinsic dimension varies smoothly with relative depth.
These experiments diagnose the local surrogate; they do not directly calibrate
geometric work to an error-width coefficient.}
\label{fig:diagnostics}
\end{figure*}

The anti-topological/raw control reverses a topological profile while also
changing to raw geometry. It tests whether a deliberately misaligned rule can
fail in the expected direction, but it is not a fully matched reversal that
isolates layer order while holding the estimator and width histogram fixed.

\paragraph{Reporting and comparison logic.}
For each rule and seed define the paired loss difference
$D_s=L_{\mathrm{rule},s}-L_{\mathrm{uniform},s}$.For the 128M and 256M screens, Fig.~\ref{fig:allocation} reports the mean paired difference $\bar D$ and descriptive normal-approximation intervals $\bar D \pm 1.96\,s_D/\sqrt{n}$, with $n=3$. Because only three paired seeds are available, these intervals are used as exploratory summaries rather than as a basis for formal hypothesis testing. These intervals are descriptive rather than a basis for selecting a definitive winner among eleven alternatives. The 440M experiment narrows the comparison to five schedules and
uses five paired seeds. Table~\ref{tab:440m-results} reports marginal seed means
and standard deviations, together with $\bar D$; marginal standard deviations
do not determine $s_D$ because they omit the within-seed covariance. Therefore,
until the five paired differences (or $s_D$, its standard error, and a
$t_4$ interval) are supplied from the run logs, the 440M comparisons are stated
as differences in mean loss rather than inferential claims of superiority. All
qualitative claims distinguish a within-protocol comparison from an absolute
comparison with values reported by another paper.

\subsection{Diagnostics of the Allocation Model}

\paragraph{Scale sensitivity.}
The first experiment asks whether work measures geometry or merely residual
scale. Raw Gromov and persistence work grow strongly in late layers and track
residual-norm growth, as expected because their pairwise distances and
filtration scales change under Euclidean rescaling. Per-token normalization
largely removes that late-depth growth. Raw mean shift also tends to rise in the
observed profiles, but Eq.~\eqref{eq:raw-shift-identity} shows that this is not a
direct metric-scale effect: it is an empirical co-variation between residual
norm and FFN-output magnitude. The contrast is visible across work families in
Fig.~\ref{fig:profiles}. We therefore treat raw work as a control, while keeping
the causal interpretation statistic-specific.

\paragraph{Local additivity and linearity.}
The joint perturbation loss is positively associated with the sum of the
corresponding single-layer losses, with $R^2\approx0.72$
(Fig.~\ref{fig:diagnostics}, left). The rank-removal curves are approximately
linear for small and moderate perturbations and become nonlinear for the most
aggressive interventions in several models. This is the regime expected for a
first-order surrogate, so $\lambda_\ell$ is estimated only from the local part
of the curve. The correlation does not by itself prove equality to the identity
line, but it supports using an additive approximation for schedule derivation.

\paragraph{Intrinsic dimension.}
TwoNN estimates stay far below the ambient hidden dimension and change smoothly
with depth (Fig.~\ref{fig:diagnostics}, right), consistent with prior evidence
that learned representations occupy low-dimensional sets
\citep{ansuini2019intrinsic}. Early layers generally have the smallest estimates,
while some larger Gemma models reach values around $4$--$5$ in the middle of the
stack. The estimates shown and used by the allocator are the unmodified TwoNN
outputs: no clipping or smoothing is applied. With the fixed expansion-order
parameter $s=3$, $\theta_\ell=\dint(\ell)/(3+\dint(\ell))$ remains in an
intermediate range and changes smoothly. The exponent therefore compresses
large differences in measured work rather than transferring them one-for-one
to width. This supports the use of an intrinsic-dimensional exponent, while not
establishing the exact rate in Eq.~\eqref{eq:rate}.

\paragraph{Sensitivity association.}
After standardization within each model, Gromov work has a stronger association
with $\lambda_\ell$ than topological work in five of seven models; the pooled
association of the current topological estimator is weak. The result argues
against claiming that topology universally predicts layer importance. It also
does not invalidate topological allocation: $\lambda_\ell$ measures the loss
consequence of an approximation error, whereas $C_\ell$ is intended to measure
the difficulty of realizing the layer map. The two factors enter the surrogate
multiplicatively and can rank estimators differently.

Taken together, the diagnostics support using a normalized, regularized profile
and the local additive objective, but they do not establish
$C_\ell=\What_\ell$. This boundary is important for interpreting the later
training experiments: a successful schedule is evidence that the proxy carries
useful ordering information, not a measurement of the true approximation
seminorm or a validation of a universal nonparametric exponent.

\subsection{Work Profiles Across Depth}

\begin{figure}[t]
\centering
\includegraphics[width=0.99\linewidth]{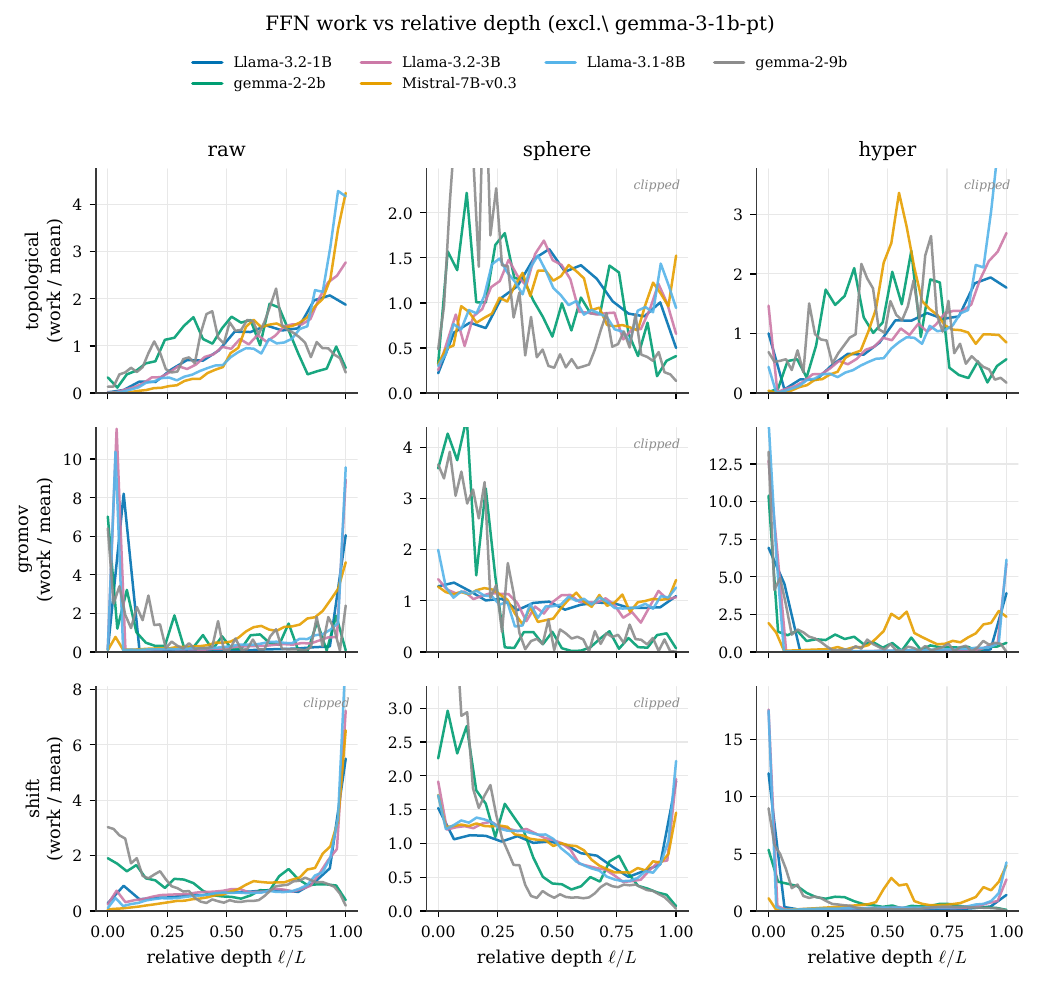}
\caption{\textbf{FFN work across relative depth.} Three work statistics (rows)
and three geometries (columns) are shown for the six models with complete
$3\times3$ profiles; each curve is normalized to unit mean. Raw Gromov and
topological work are directly scale-sensitive and rise with residual scale.
Raw shift also rises empirically in these checkpoints, but is not algebraically
rescaled by residual norm. Spherical and hyperbolic work are predominantly
front-loaded and can depart from a fixed monotone taper.}
\label{fig:profiles}
\end{figure}

Figure~\ref{fig:profiles} compares the depth profile after normalizing each
model/statistic curve to unit mean. In raw geometry, topological and Gromov work
tend to increase near the end of the network, consistent with direct
scale-sensitivity of pairwise distances and persistence filtrations. Raw shift
also tends to increase, but this is interpreted only as empirical co-variation
of FFN-output magnitude with residual norm. Under spherical and hyperbolic
geometry, the ordering mostly reverses: work is high in the earliest layers,
decreases through the middle of the stack, and sometimes rises mildly near the
end. The qualitative front-loading transfers across model families, but the
local deviations differ by estimator and checkpoint. Thus the measurements
support the broad direction identified by TLM while supplying a layer-specific,
potentially non-monotone profile rather than imposing a cosine shape.

Relative depth $\ell/L$ is used because the checkpoints have different numbers
of blocks. Unit-mean normalization removes the arbitrary scale of each
estimator and makes the experiment about allocation shape; the total width is
set later by the budget constraint. The cross-family agreement is strongest for
the early peak and weaker for fine structure near the middle and final blocks.
That pattern motivates transferring a smoothed profile rather than copying
individual noisy layer values. It also explains why geometry and tapering are
complementary: both favor earlier capacity, but only the measured profile can
preserve repeatable departures from monotonicity.

\begin{figure}[h]
\centering
\includegraphics[width=0.92\linewidth]{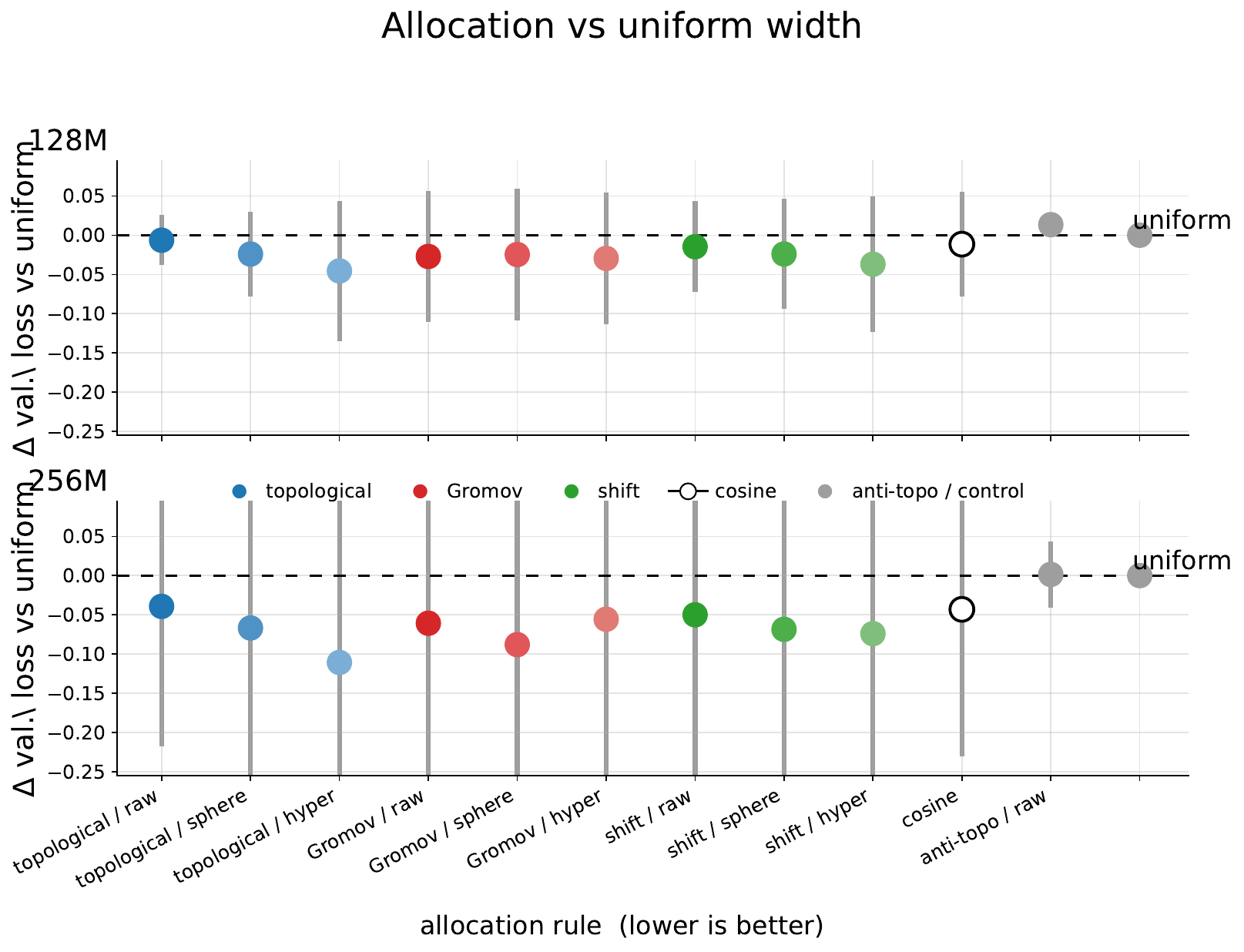}
\caption{\textbf{Allocation rules compared with uniform width.} Paired change
in final validation loss for the 128M and 256M iso-parameter runs, negative is
better. Points are means over three paired seeds and bars are 95\% confidence
intervals.}
\label{fig:allocation}
\end{figure}

\subsection{Paired 128M and 256M Training}

Figure~\ref{fig:allocation} evaluates the full $3\times3$ work design space at
two tractable model sizes. Most schedules based on normalized geometry have
negative mean deltas relative to their seed-matched uniform baselines.
Topological work in hyperbolic geometry has the lowest mean validation loss
among the tested configurations; Gromov/spherical and shift/hyperbolic also
improve on average. The TLM-style cosine schedule remains close to uniform,
while the anti-topological/raw control is centered near or above zero.

These runs show that the profile contains information beyond parameter count:
every rule uses the same FFN budget, and pairing removes initialization and data
order from the within-seed comparison. At the same time, three seeds are too few
for a definitive estimator ranking. The 95\% intervals overlap zero and one
another, so the small-scale result is best interpreted as a broad screen that
identifies normalized schedules for the more focused larger experiment. The
fact that Gromov is strongest in the sensitivity diagnostic while topological
work has the best training mean reinforces that sensitivity association and
allocation quality are related but non-equivalent tests.

These results covering all combinations of statistics and geometries allow to disentangle two design choices, the first one is whether scale normalization improves allocation, and the second is how much correspondence and relational structure the estimator should preserve. Normalized geometries perform more consistently favorably than raw geometry, whereas no single work statistic is best across all settings. The results therefore provide stronger support for normalization as a general design principle than for any particular topological or transport-based estimator. Because this comparison is exploratory, the larger experiment evaluates a representative set of leading rules rather than choosing a separate winner for each seed or model size.


\subsection{440M Comparison with TLM}

TLM uses a 440M Transformer to select a hand-designed taper and reports that its
best cosine schedule improves validation perplexity from $16.28$ to $14.44$ in
its training setup \citep{tlm2026}. We include the same model scale and treat the
cosine schedule as a direct baseline within our own protocol. Unlike an absolute
cross-paper comparison, this within-protocol experiment controls the data,
tokenizer, optimization, parameter count, and random seeds across allocation
rules.

Table~\ref{tab:440m-results} gives the focused descriptive comparison in this
study. The TLM-style cosine taper has mean validation loss $0.003$ below uniform.
Topological/hyperbolic allocation has mean loss $0.019$ below uniform,
approximately $6.3\times$ the mean reduction of cosine, while Gromov/spherical
has mean loss $0.016$ below uniform. The two geometry-derived schedules
therefore have lower mean validation loss than the prescribed cosine baseline
in the same five-seed experiment. The anti-topological/raw control moves in the
opposite direction and has higher mean loss and perplexity. These statements do
not establish statistical separation because the paired-difference variability
is not reported in the current table.

\begin{table}[h]
\centering
\small
\setlength{\tabcolsep}{3pt}
\caption{\textbf{Five-seed 440M iso-parameter results.} Validation loss is
reported as the marginal mean $\pm$ marginal standard deviation over five paired
seeds. Perplexity is $\exp(\text{mean loss})$, and $\Delta=\bar D$ is the mean
paired loss difference relative to uniform (equivalently, the arithmetic
difference of means under complete pairing). The table does not contain the
standard deviation or confidence interval of $D_s$, so it supports descriptive
mean comparisons rather than an inferential superiority claim. Lower is better.}
\label{tab:440m-results}
\begin{tabular}{@{}lccc@{}}
\toprule
Rule & Val. loss $\downarrow$ & PPL $\downarrow$ & $\Delta$ loss $\downarrow$ \\
\midrule
uniform & $3.449\pm0.022$ & $31.47$ & - \\
TLM cosine & $3.446\pm0.022$ & $31.37$ & $-0.003$ \\
topological / hyper & $\mathbf{3.430\pm0.023}$ & $\mathbf{30.88}$ & $\mathbf{-0.019}$ \\
Gromov / sphere & $3.433\pm0.028$ & $30.97$ & $-0.016$ \\
anti-topological / raw & $3.515\pm0.098$ & $33.62$ & $+0.066$ \\
\bottomrule
\end{tabular}
\end{table}

The ratio is a within-protocol descriptive statement, not a claim that the
absolute losses are directly comparable with the perplexities reported by TLM:
the two studies use different training pipelines. The relevant connection is
methodological. TLM shows that front-loading capacity can work; the present
matched experiment indicates that an activation-derived profile can have a
larger mean reduction than a fixed cosine. Five paired seeds are more
informative than the three-seed 128M/256M screen, but sample count alone does not
establish stability, and the paired interval is needed for inferential claims.

\section{Discussion}
\label{sec:discussion}

The experiments support two levels of conclusion. First, FFN capacity should not
be assumed uniform. This agrees with TLM and with the observation that layer
updates change their role across depth. Second, a fixed monotone taper is not the
only useful alternative. Normalized work profiles recover the broad early-heavy
shape but preserve model- and layer-specific structure, and the five-seed 440M
comparison shows a substantially larger mean loss reduction for the
geometry-derived schedules than for cosine under identical conditions. This is
a descriptive statement about means until paired-difference uncertainty is
reported.

The results do not establish one universally superior geometry. Gromov work is
more consistently associated with perturbation sensitivity, whereas
topological/hyperbolic allocation obtains the best training mean. This split is
plausible under the theory: sensitivity $\lambda_\ell$ and approximation
coefficient $C_\ell$ play different roles, and the training outcome also depends
on profile smoothing, rounding, and interactions introduced by retraining. A
conservative reading is therefore that normalized geometric change is useful,
not that persistent homology is always the correct estimator.

The central limitation is the absence of a direct layerwise width sweep. Such an
experiment would test whether $\What_\ell$ predicts the coefficient of an
error-width curve, rather than relying on sensitivity association and
end-to-end outcomes. The geometric statistics are also finite-sample estimates as
persistence and optimal transport can have different variance and computational
cost. The fixed-radius hyperbolic embedding is only a nonlinear transform of
angular distance and does not test a learned radial hierarchy. The schedule is
measured on a reference model and transferred to a target model trained from
scratch, so profile stability under architectural change is assumed. Finally,
the anti-work control changes both direction and geometry, and therefore does
not isolate alignment as cleanly as a matched reversed or permuted normalized
profile would. The three-seed small-scale intervals remain wide. The five-seed
440M means are more informative, but the current table omits uncertainty of the
paired differences and therefore does not justify inferential or fine-grained
claims about estimator ordering.

\section{Conclusion}
Geometric work provides a principled way to replace a global FFN-width constant
with a fixed-budget, layer-dependent schedule. Raw pairwise and topological
statistics are scale-sensitive, whereas raw mean shift measures FFN-output
magnitude and only empirically covaries with residual scale. Normalized metric
and topological changes reveal a transferable early-heavy profile. Across
paired training runs, geometry-derived schedules have lower mean validation
loss than uniform width in several comparisons and, at 440M, have a larger mean
reduction than the TLM-style cosine baseline within the same five-seed protocol.
Paired-difference intervals are required before interpreting these mean gaps as
statistical separation. These findings position activation-derived capacity
allocation as a practical extension of tapering and motivate direct calibration
of geometric work to layerwise approximation demand.

\clearpage
\appendix
\section{Experimental Details and How the Figures Are Produced}
\label{app:experimental-details}

This appendix documents the experiments of the main paper in reproducible detail and specifies exactly how each figure and table is produced: what quantity is on each axis, how it is estimated, what normalization and aggregation are applied, and what the error bars mean. It contains no new claims; it expands the experimental protocols so that every panel can be regenerated from the released code. Section~\ref{app:shared-setup} gives the shared setup (data, models, estimators, budgets, seeds, hardware, and statistics); Section~\ref{app:figure-production} describes each figure and table panel by panel.

\subsection{Experimental Details}
\label{app:shared-setup}

\subsubsection{Data and Preprocessing}
All measurements and all training use OpenWebText \citep{gokaslan2019openwebtext}. For the pretrained-model diagnostics we stream the corpus and pass a held-out sample through each network in inference mode (bfloat16, scaled dot-product attention, no gradient computation), recording for every FFN sublayer the post-attention input cloud $Z_\ell$ and the post-FFN output cloud $H_\ell$. Raw statistics use the original residual states; spherical and hyperbolic statistics use the corresponding normalized embeddings. Persistence and \GromovW{} statistics for a given layer use the same fixed token subsample for $Z_\ell$ and $H_\ell$, so that a measured difference reflects the action of the sublayer rather than a change of examples. For the from-scratch runs, OpenWebText is tokenized with a byte-level BPE tokenizer \citep{radford2019language} and packed into a flat token stream; training and validation shards are disjoint at the document level.

\subsubsection{Pretrained Models}
Seven decoder-only checkpoints spanning three families and a range of sizes are used: Llama-3.2-1B, Llama-3.2-3B, and Llama-3.1-8B \citep{dubey2024llama}; Gemma-2-2B and Gemma-2-9B \citep{gemmateam2024gemma2}; Gemma-3-1B \citep{gemmateam2025gemma3}; and Mistral-7B-v0.3 \citep{jiang2023mistral}. Depth profiles are plotted against relative depth $\ell/L$ because the checkpoints have different numbers of blocks. Six of the seven have complete $3\times3$ work profiles and are the ones shown in the profile grid.

\subsubsection{Work Statistics and Geometries}
For each geometry we compute three notions of work on the matched clouds. With $u(z)=z/\lVert z\rVert$, raw geometry uses the identity map and Euclidean distance; spherical geometry uses $u(z)$ and the angular distance $\arccos\langle u(z),u(z')\rangle$; hyperbolic geometry uses a fixed monotone radial transform into the unit Poincar\'e ball with its geodesic distance. The three statistics are the topological work
$\Wtop_\ell=W_2(\PD_1(Z_\ell),\PD_1(H_\ell))$, where $W_2$ is the unsquared 2-Wasserstein distance between degree-one persistence diagrams under the $\ell_\infty$ birth--death ground metric (diagonal matching handles empty diagrams); the metric work $\Wmet_\ell$, a fixed-regularization \GromovW{} discrepancy between the intra-cloud distance matrices of $Z_\ell$ and $H_\ell$; and the shift work $\Wsh_\ell$, the mean transported distance between matched input and output points under the chosen geometry. Persistence is computed from the Vietoris--Rips filtration of the distance matrix at degree one. Every curve in the profile figure is normalized to unit mean over depth, which removes the arbitrary scale of each estimator and makes the comparison about allocation shape; the total width is set separately by the budget constraint.

\subsubsection{Perturbation Diagnostics}
Layer sensitivity $\lambda_\ell$ is estimated by low-rank FFN perturbations. For a layer we remove increasing rank fractions from the FFN branch and record the resulting increase in validation loss; $\lambda_\ell$ is the slope of the small-perturbation (locally linear) part of that curve. To probe the local additive-loss assumption, sets of layers are perturbed jointly and the joint loss increase is compared against the sum of the corresponding single-layer increases. These diagnostics do not train alternative FFN widths and therefore do not calibrate geometric work to an error--width coefficient; they test the local surrogate only.

\subsubsection{Intrinsic Dimension}
Intrinsic dimension is estimated independently at each relative depth with TwoNN \citep{facco2017twonn}. For each point $x_i$ in the layer-$\ell$ cloud, with $r_{i,1}$ and $r_{i,2}$ its first two nearest-neighbour distances, let $\mu_i=r_{i,2}/r_{i,1}$. Under the TwoNN local-uniformity model $F(\mu)=1-\mu^{-d}$, $\dint(\ell)$ is the origin-slope of $-\log(1-\widehat F(\mu_{(i)}))$ against $\log\mu_{(i)}$ on the sorted ratios. The TwoNN outputs are used directly: no clipping, winsorization, or smoothing. With the fixed global expansion order $s=3$, the exponent is $\theta_\ell=\dint(\ell)/(3+\dint(\ell))$ and $a_\ell=3/\dint(\ell)$.

\subsubsection{From-Scratch Allocation Runs}
We train approximately 128M and 256M Transformers on tokenized OpenWebText. The 128M/256M design space is the full set of nine work/geometry combinations, plus uniform width, the TLM-style cosine taper, and an anti-topological/raw control (eleven rules in total). Each size uses three paired seeds: within a seed, all rules share initialization and data order, so a within-seed comparison removes those nuisance factors. Every rule at a given size has the same total parameter count, the same token budget, and the same AdamW configuration, so the comparison is iso-parameter and iso-token; the final validation loss is the primary endpoint. The 440M experiment spends a larger compute budget on five paired seeds for a focused set of five rules: uniform, cosine, topological/hyperbolic, Gromov/spherical, and anti-topological/raw. The anti-topological/raw control reverses a topological profile and simultaneously switches to raw geometry; it tests whether a deliberately misaligned rule fails in the expected direction, but it is not a fully matched reversal that isolates layer order while holding the estimator and width histogram fixed.

\subsubsection{Width Schedules from a Reference Profile}
A measured, unit-mean work profile is smoothed across relative depth and mapped to a width schedule by the budget-optimal rule of the main paper, so that $\sum_\ell w_\ell$ equals the uniform budget. Widths are quantized to the architecture's width granularity and floored at a minimum width; the quantization residual is redistributed so that the total width matches the uniform baseline exactly, which is what makes every rule iso-parameter by construction rather than approximately. Profiles are transferred by relative depth $\ell/L$ because the reference and target models differ in depth, and a smoothed profile is transferred rather than individual noisy layer values.

\subsubsection{Computing Infrastructure}
All experiments run on a single NVIDIA RTX A6000 (GA102; 48 GB GDDR6 with ECC; approximately 768 GB/s memory bandwidth; 300 W board power; PCIe 3.0 $\times16$ in this host); no multi-GPU, model-parallel, or sharded training is used. At the training batch used, the largest configuration occupies roughly 26 GiB of the 48 GiB available, so no activation checkpointing or offloading is required. On a CUDA out-of-memory event the trainer halves the microbatch and raises gradient accumulation to keep the tokens per optimizer step unchanged, turning memory pressure into reduced throughput rather than a failed run. Between runs, GPU memory is released and downloaded model and dataset artifacts are purged from the local cache; runs are checkpointed and resumable.

\subsubsection{Randomness and Seeds}
Python, NumPy, and PyTorch seeds (including all CUDA devices) are set at the start of each run. Within a seed, initialization and data order are shared across all allocation rules, so each rule-versus-uniform comparison is exactly paired and isolates the width profile. The 128M and 256M screens use three paired seeds each; the 440M experiment uses five paired seeds. The pretrained-model diagnostics draw a single fixed token subsample per layer.

\subsubsection{Evaluation Metrics and Reporting}
The endpoint is final validation loss (cross-entropy in nats); perplexity is $\exp(\text{mean loss})$. For a rule and seed we define the paired difference $D_s=L_{\mathrm{rule},s}-L_{\mathrm{uniform},s}$; $\bar D$ is its mean over seeds. For the 128M/256M screens we report $\bar D$ with descriptive normal-approximation intervals $\bar D\pm1.96s_D/\sqrt n$ at $n=3$; because only three paired seeds are available, these intervals are exploratory summaries, not a basis for formal testing or for selecting a single winner among eleven alternatives. For the 440M experiment the table reports marginal seed means and standard deviations together with $\bar D$; marginal standard deviations do not determine the standard deviation of $D_s$ because they omit the within-seed covariance, so the 440M comparisons are stated as differences in mean loss rather than as inferential superiority claims until the paired-difference variability is supplied from the run logs. All qualitative claims distinguish a within-protocol comparison from an absolute comparison against numbers reported by another paper.

\subsection{How Each Figure and Table Is Produced}
\label{app:figure-production}

\subsubsection{Diagnostics of the Allocation Surrogate}
The three-panel diagnostic figure (Figure~\ref{fig:diagnostics}; source identifier \texttt{fig\_assumptions\_abc}) visualizes the three assumptions behind the width rule; each panel is a scatter or profile over the pretrained models of Section~\ref{app:shared-setup}.

\paragraph{Panel (a): local additivity.}
Each point is one joint perturbation of a set of layers. Its $x$-value is the sum of the single-layer validation-loss increases of the layers in that set (each measured separately as above); its $y$-value is the validation-loss increase when those layers are perturbed jointly. The identity line $y=x$ is exact additivity. The reported association is $R^2\approx0.72$; the panel shows a positive association with the identity line rather than exact equality, which is the evidence used to justify an additive objective for schedule derivation. Because a point is a set of layers, the panel contains many more points than there are layers, and the perturbation strength is held in the locally linear regime (panel (b)) so that the scatter reflects cross-layer interaction rather than within-layer nonlinearity.

\paragraph{Panel (b): local linearity of rank removal.}
For a layer, the curve plots the validation-loss increase (vertical axis) against the removed rank fraction of the FFN branch (horizontal axis). Curves are shown per layer and are approximately linear for small and moderate removal, becoming nonlinear for the most aggressive interventions in several models. $\lambda_\ell$ is read as the slope of the linear part only; the panel's purpose is to show that this linear region exists and to delimit the perturbation strength used in panel (a).

\paragraph{Panel (c): intrinsic dimension versus depth.}
For each model, the curve plots the TwoNN estimate $\dint(\ell)$ (vertical axis) against relative depth $\ell/L$ (horizontal axis). Values are the unmodified TwoNN outputs. The panel shows that $\dint$ stays far below the ambient hidden dimension, changes smoothly with depth, is smallest in early layers, and reaches roughly 4--5 in the middle of the larger Gemma stacks. This motivates the intrinsic-dimensional exponent $\theta_\ell=\dint(\ell)/(3+\dint(\ell))$, which compresses large differences in measured work rather than transferring them one-for-one to width.

\paragraph{Optional companion panels.}
Where the additivity and dimension panels are shown separately (source identifiers \texttt{fig\_e\_additivity} and \texttt{fig\_e\_dim}), the axes and estimators are identical to panels (a) and (c) above.

\subsubsection{FFN Work Across Relative Depth}
The profile grid (Figure~\ref{fig:profiles}; source identifier \texttt{fig\_ffn\_grid\_clean}) is a $3\times3$ array: rows are the three work statistics ($\Wtop$, $\Wmet$, and $\Wsh$) and columns are the three geometries (raw, spherical, and hyperbolic). Within each cell, one curve per model plots that statistic against relative depth $\ell/L$; only the six models with complete $3\times3$ profiles are drawn. Every curve is normalized to unit mean over depth, so the panels compare shape, not magnitude. The figure is read for two contrasts. First, raw Gromov and raw topological work rise toward the end of the network, tracking residual-norm growth, because their pairwise distances and filtration scales change under Euclidean rescaling; raw shift also rises, but as an empirical co-variation of FFN-output magnitude with residual norm rather than an algebraic rescaling. Second, under spherical and hyperbolic geometry the ordering largely reverses: work is front-loaded (high in the earliest layers, decreasing through the middle, occasionally rising mildly near the end) and can depart from a fixed monotone taper. The cross-family agreement is strongest for the early peak and weakest for fine structure near the middle and final blocks, which is why a smoothed profile is transferred rather than individual noisy layer values.

\subsubsection{Allocation Rules Versus Uniform (128M and 256M)}
The forest plot (Figure~\ref{fig:allocation}; source identifier \texttt{figure6\_allocation\_vs\_uniform\_forest}) evaluates the full eleven-rule design space at the two smaller sizes. Each row is one allocation rule. The plotted point is the mean paired change in final validation loss relative to the seed-matched uniform baseline, $\bar D=\operatorname{mean}_s(L_{\mathrm{rule},s}-L_{\mathrm{uniform},s})$, over three paired seeds; more negative is better. The bars are the descriptive 95\% normal-approximation intervals $\bar D\pm1.96s_D/\sqrt3$ described above. The 128M and 256M sizes are shown as grouped rows so that the same rule can be compared across the two scales. Because the intervals overlap zero and one another, the panel is read as an exploratory screen that identifies normalized-geometry schedules as consistently favorable and the anti-topological/raw control as centered near or above zero, not as a definitive ranking of estimators.

\subsubsection{Five-Seed 440M Comparison with TLM}
Table~\ref{tab:440m-results} reports, for five rules (uniform, cosine, topological/hyperbolic, Gromov/spherical, and anti-topological/raw), the marginal validation-loss mean $\pm$ marginal standard deviation over five paired seeds, the perplexity $\exp(\text{mean loss})$, and $\Delta=\bar D$, the mean paired loss difference relative to uniform. Under complete pairing $\Delta$ equals the arithmetic difference of the two reported means. The table deliberately omits the standard deviation and confidence interval of the paired difference $D_s$, because the marginal standard deviation does not determine it (the within-seed covariance is not reported); accordingly the 440M results support descriptive mean comparisons and are not stated as inferential superiority claims. The cosine schedule is treated as a direct in-protocol baseline for the same 440M scale used by TLM, with data, tokenizer, optimization, parameter count, and seeds controlled across rules; the comparison is within-protocol and is not an absolute comparison against the perplexities reported by TLM, whose training pipeline differs.

\bibliographystyle{plainnat}
\bibliography{references}

\end{document}